\documentclass{new_tlp}
\usepackage{mathptmx}
\usepackage{algpseudocode}
\usepackage{algorithmicx}
\usepackage{listings}
\usepackage{caption}
\usepackage{url}
\usepackage[all]{xy}
\usepackage{amsmath}

\usepackage{thmtools}
\renewcommand\thmcontinues[1]{Continued}

\usepackage{listings}
  \title[Nonground ALP with Probabilistic Integrity Constraints]
        {Nonground Abductive Logic Programming with Probabilistic Integrity Constraints}

\author[Bellodi et al.]
{Elena Bellodi, Marco Gavanelli, Riccardo Zese, Evelina Lamma\\
Department of Engineering -  University of Ferrara
\and Fabrizio Riguzzi\\
Department of Mathematics and Computer Science - 
University of Ferrara}

\jdate{10 May 2021}
\pubyear{2021}
\pagerange{\pageref{firstpage}--\pageref{lastpage}}
\doi{S1471068401001193}

\newtheorem{definition}{Definition}
\newtheorem{theorem}{Theorem}

\newtheorem{example}{Example}

\usepackage{todonotes}

\usepackage{acronym}
\newacro{IC}{\textit{Integrity Constraint}}
\newacro{CHR}{Constraint Handling Rules}
\newacro{CLP}{Constraint Logic Programming}
\newacro{PIC}{\textit{Probabilistic Integrity Constraint}}
\newacro{BDD}{Boolean Decision Diagram}
\newacro{LP}{Logic Programming}
\newacro{ALP}{Abductive Logic Programming}

\usepackage{sciff}

\newcommand{\AbducibleSet}{{\ensuremath{\mathcal A}}}

\newcommand{\WorldSet}{{\ensuremath{{E}}}}
\newcommand{\worldset}{explanation}
\newcommand{\aworldset}{an explanation}
\newcommand{\worldsetCHR}{{\ensuremath{expl}}} 
\newcommand{\SetWorldSet}{{\ensuremath{{\mathcal E}}}}
\newcommand{\setworldset}{set of \worldset s}

\newcommand{\stripProb}[1]{{\ensuremath{#1}|_{\not{p}}}} 
\newcommand{\taut}[1]{{\ensuremath{\top(#1)}}} 

\newcommand{\World}{{\ensuremath{\mathcal W}}}
\newcommand{\world}{world}
\newcommand{\SetWorld}{{\ensuremath{{\mathcal W}_{T}}}}

\newcommand{\ic}{{\ensuremath{ic}}}  
\newcommand{\icset}{{\ensuremath{\textbf{IC}}}}  
\newcommand{\icsetnonprob}{{\ensuremath{\icset^{np}}}}  
\newcommand{\icsetprob}{{\ensuremath{\icset^p}}}  

\newcommand{\fabd}[1]{{\ensuremath{\texttt{\textit{#1}}}}}
\newcommand{\cabd}[2]{{\ensuremath{\texttt{\textit{#1}#2}}}}

\begin{document}
\nocite{*}

\label{firstpage}

\maketitle

  \begin{abstract}

Uncertain information is being taken into account in an increasing number
of application fields. In the meantime, abduction has been proved a
powerful tool for handling hypothetical reasoning and incomplete
knowledge.
Probabilistic logical models are a suitable framework to handle uncertain
information, and in the last decade many probabilistic logical languages
have been proposed, as well as inference and learning systems for them.
In the realm of Abductive Logic Programming (ALP), a variety of proof
procedures have been defined as
well. 
In this paper, we consider a richer logic language, coping with
probabilistic abduction with  variables. In particular, we
consider an ALP program enriched with 
integrity constraints \textit{\`{a} la}
IFF, possibly annotated with a probability value. 
We first present the overall abductive language, and its semantics
according to the Distribution Semantics. 
We then introduce a proof procedure, obtained by extending one previously presented, 
and prove its soundness and completeness.
This paper is under consideration for acceptance in TPLP.

%

  \end{abstract}



\section{Introduction}

Reasoning in uncertain domains is a common task for humans, 
and the human brain also has the capability to explore different scenarios by considering a variety of possible hypotheses, in order to take a decision. The ability of the human brain, and human expertise in specific domains, meant that humans were not replaceable by a machine in their reasoning tasks, so far. Nonetheless, in the last decade, the huge increase of available data and knowledge  in many domains (e.g., in medicine, science, physics, etc.), often in a form that can be processed automatically, strongly pushes towards forms of automatic reasoning able to cope with uncertainty, probabilities and hypotheses, also in order to have reasoning systems facing humans, and to achieve a more reproducible (and verifiable) behaviour. This is definitively  a strong commitment for a trustworthy Artificial Intelligence. 
\ac{LP} is a powerful class of languages to be a candidate for this purpose. The language itself is human-readable, and knowledge expressed in this class of languages can be validated by humans.
Standard LP syntax, and LP-based reasoning is the base for a variety of more expressive languages, and proof procedures.
In particular, Probabilistic Logic Programming (PLP)~\cite{PILP} 
languages are simple yet powerful enough to
represent different scenarios~\cite{azzolini2019studying,NguRig17-IMAKE-BC}. Several of these languages are based on the \textit{distribution semantics}~\cite{DBLP:conf/iclp/Sato95}, such as PRISM~\cite{DBLP:conf/iclp/Sato95}, LPAD~\cite{VenVer04-ICLP04-IC}, and ProbLog~\cite{DBLP:conf/ijcai/RaedtKT07}.

In the meantime, \ac{ALP}~\cite{ALP_Handbook} has been proven very effective for hypothetical reasoning and for formalizing a variety of domains and applications, ranging from diagnosis to 
societies of agents and accountable protocols for
multi-agent systems, commitments and normative systems, 
 web service choreographies. 
\ac{ALP} is based on a declarative (model-theoretic) semantics and equipped with an operational semantics
in terms of a proof-procedure.
The IFF proof-procedure  was proposed by \citeN{IFF} to support abductive reasoning
also in presence of non-ground abducible literals.

In this paper we consider a richer logic language, coping with \textit{probabilistic abduction with 
variables}. In particular, we consider an ALP program featuring also 
\acp{IC}
similar to those offered by IFF, extended by the possibility of annotating them  with a probability value, that makes it possible to handle uncertainty of real world domains. 
Probabilistic integrity constraints were defined by  \citeN{RigBelZesAlbLam20-ML-IJ}: programs containing such constraints are called Probabilistic Constraint Logic Theories (PCLTs) and may be  learned  directly from data by means of  PASCAL (``ProbAbiliStic inductive ConstrAint Logic''), a system that learns both their structure and  parameters  from interpretations.
However, a system able to reason about these integrity constraints is still missing.
Consider the following example.

\begin{example}[label=exa:running]
Several years ago, a murder in Italy captured the attention of the population: a woman was murdered, and the main indicted person was her husband.
The collected evidence included the following facts:
the woman was killed in the house where she lived with her husband ({\tt house1});
a pillow stained with the blood of the victim was found in another house ({\tt house2}) some hundreds of km away;
the husband had the keys of this second house.\\
We can represent the facts listed above in the following knowledge base:
\begin{lstlisting}
has_keys(husband,house1).
has_keys(husband,house2).
\end{lstlisting}
The goal is to find the murderer $M$, i.e., the person who entered both houses and killed the victim:
\begin{lstlisting}
$G =  \cabd{enter}{(M,house1)},\cabd{killed}{(M,woman)},\cabd{enter}{(M,house2)}.$
\end{lstlisting}
Predicates \cabd{killed} and \cabd{enter} are not known in the knowledge base, and they 
must be hypothesized: in \ac{ALP} they are considered abducibles (in this paper, abducibles are in \fabd{italic}).
Notice that this problem requires non-ground abduction, since the murderer $M$ is unknown, and it is not even possible to list all the possible murderers.

The relationship between having the keys and entering a house can be stated through an integrity constraint
%
\begin{equation}
\label{eq:exampleMurder_IC_NonProb}
\tag{$ic_1$}
\cabd{enter}{(P,H)} \rightarrow has\_keys(P,H).
\end{equation}
saying that if a person {\tt P} enters a house {\tt H},  (s)he must have the keys.

However, the information encoded in $ic_1$ is not 100\% sure: a person could also enter the house without having the keys, e.g., by breaking a window or
picking the lock of the door.
The encoding would be more faithful to reality if a probability was associated to \ref{eq:exampleMurder_IC_NonProb}.
The probability that \ref{eq:exampleMurder_IC_NonProb} does not hold would be quite low, since
entering with the keys is much easier than with unlawful methods,
in which the intruder could be noticed and arrested.

\end{example}

Most previous works proposing probabilistic abductive logic programming have considered only abduction of ground atoms~\cite{DBLP:journals/ai/Poole93,arvanitis2006abduction,DBLP:conf/ijcai/InoueSIKN09,kate2009rj,conf/ijcai/Raghavan11,TurliucEtAl2013,DBLP:conf/cilc/RotellaF13}.
To the best of our knowledge, only  \citeN{DBLP:series/lncs/Christiansen08} focused on probabilistic non-ground abduction,
and this clearly extends the expressiveness of the language, as well as the answer capabilities of the proof-procedure:
as a matter of fact, non-ground abduction can provide answers in which not all parts of the answer are completely defined,
and in which possible hypotheses, for some given evidence, are assumed without complete knowledge. This more closely resembles
the capabilities of the human brain, that is able to hypothesize the existence of some action, force or individual causing an effect
even without complete knowledge of whom, or which force is responsible.
On the other hand, the proof-procedure by \citeN{DBLP:series/lncs/Christiansen08} considers only a limited form of negation and
assigns probabilities only to abducibles.
In our work, instead, we extend ALP to give the possibility to probabilistically annotate integrity constraints to enrich the standard semantics. This leads to a probabilistic ALP reasoning system that clearly improves on existing systems such as the one by Christiansen, since annotating ICs also allows one to obtain the same effect of adding probabilities to abducibles. Moreover,   machine learning systems already exist that learn probabilistic integrity constraints from data in the same form proposed in this paper, such as the aforementioned PASCAL, however,  no abductive logic programming reasoning system existed before to exploit the learned constraints: in this paper we fill this gap.

The paper is organized as follows. Sect. \ref{sec:bg} introduces information about the IFF semantics necessary for understanding Sect. \ref{sec:iffprob} and \ref{sec:iffprob2} that discuss language, syntax, and semantics of the proposed probabilistic non-ground abduction proof-procedure, called \sciffprob. Sect. \ref{sec:sound-compl} proves soundness and completeness of \sciffprob, while Sect. \ref{sec:impl} describes its  implementation and 
Sect.~\ref{sec:app_ex} discusses some real world application examples. Sect. \ref{sec:exp} shows preliminary scalability tests. Sect.~\ref{sec:rw} presents related work and Sect.~\ref{sec:conc} concludes the paper.

\section{Background}
\label{sec:bg}

\paragraph{\textbf{\sciff\ Declarative Semantics.}}


A \sciff\ \cite{IFF} program is a triple $\triple{\SOKB}{\icset}{\AbducibleSet}$.
$\AbducibleSet$ is a set of \emph{abducible predicates}, or simply abducibles. 
An abducible is a predicate about which it is possible to make assumptions, such as about its truth. An abducible atom is an atom built on an abducible predicate.
In this work, abduced literals can contain variables, that are implicitly existentially quantified.
\SOKB\ is a
set of logic programming clauses of the form
$$h \leftarrow b_1,\dots,b_m$$
where $m\geq 0$, and each $b_i$ with $1\leq i\leq m$ is 
a literal (i.e., an atom or its negation), 
while $h$ is an atom that cannot be built on a predicate in $\AbducibleSet$. $h$ is called \emph{head}, while $b_1,\dots,b_m$ is called the \emph{body} of the clause.

\icset\ is a set of implications, called 
\acfp{IC}.
Each $\ic\in \icset$ has the form 
$$ b_1,\dots,b_n \then H_1\vee\dots\vee H_k$$
where $b_1,\dots,b_n$ is a conjunction of 
atoms 
 called $Body$ of the \ac{IC} \ic, 
while the disjunction $H_1\vee\dots\vee H_k$ is called the $Head$. Each $H_i$ with $1\leq i\leq k$ is a conjunction of literals $h_1,\dots,h_l$. 

A goal $G$ is a conjunction of literals of the form 
$g_1,\dots,g_m$
with $m\geq1$.



To define the abductive semantics, we need to recall some definitions. An \ac{IC}, clause or abducible is \emph{ground} if it does not
contain variables. A \emph{substitution} $\theta$ is an assignment of
variables to terms: $\theta=\{V_1/t_1,\ldots,V_n/t_n\}$. The
application of a substitution
	$\theta=\{V_1/t_1,\ldots,V_n/t_n\}$ to an \text{\ac{IC}} $\ic$, indicated with $\ic\theta$, is the replacement of each
variable $V_i$ appearing in $\ic$ and in $\theta$ with $t_i$. $\ic\theta$
is called an \emph{instance} of $\ic$. $\theta$ is \emph{grounding}
for $\ic$ if $\ic\theta$ is ground. The same applies to clauses and abducibles as well.

The \sciff\ abductive semantics defines 
an abductive answer to a goal $G$ as a pair $(\Delta,\theta)$, where $\Delta$ is a set of abducible atoms  and $\theta$ is a substitution for the variables contained in $G$,
 such that 
	\begin{equation} \label{eq:entailIC}
			\begin{split}
		\SOKB  \cup \Delta &\models \icset \>\>\> \qquad \text{and}\\
		\SOKB  \cup \Delta &\models G\theta
		\end{split}
	\end{equation}
	where $\models$ is entailment according to the 3-valued completion semantics \cite{Kunen87},
	i.e., we require that every 3-valued models of the completion \cite{Cla78} of $\SOKB \cup \Delta$ is also a 3-valued model of \ac{IC}.
	In such a case, we write $ALP \models^\Delta G$.
Variables in $G$ are free, the remaining variables in $\Delta$ are existentially quantified, while
variables in \SOKB\ and \icset\ are implicitly universally quantified with scope the entire implication.
	
\citeN{RigBelZesAlbLam20-ML-IJ} introduce (although not in the context of \ac{ALP}) the concept of \ac{PIC}
 \begin{equation}\label{eq:PIC}
	p_i \ ::\ Body \then Head
\end{equation}
where  $p_i$ is a probability $\in [0,1]$.

\section{\sciffprob\ Syntax and Declarative Semantics}
\label{sec:iffprob}

\sciffprob\ programs define a probability distribution over \sciff\ programs inspired by the distribution semantics proposed in the field of  Probabilistic Logic Programming~\cite{DBLP:conf/iclp/Sato95}. 
A \sciffprob\ program, similarly to \sciff,
is a triple $ALP_p = \langle \SOKB,  \icset, \AbducibleSet\rangle$, 
however, in this case, the set  $\icset$ is defined as $\icsetprob\cup\icsetnonprob$, where  $\icsetprob$ is a set of \acp{PIC} as in Eq.~\ref{eq:PIC}
while $\icsetnonprob$ is a set of (non-probabilistic) \acp{IC} as in the \sciff. 

It is worth noting that associating probabilities to \acp{IC} makes \sciffprob\ more general than other frameworks \cite{DBLP:series/lncs/Christiansen08}
in which probabilities are attached to abducibles,
since the fact that abducible $\fabd{a}$ has probability $p$ can be simply expressed by the \ac{PIC} $$ (1-p):: \fabd{a} \rightarrow \mathit{false}.$$


A \sciffprob\ program $T=\triple{\SOKB}{\icset}{\AbducibleSet}$ defines a probability distribution over \sciff\ programs
called \emph{\world s} where the constraints set includes each $\ic\in\icsetnonprob$, and some instances of any $\ic\in\icsetprob$ included with probability $p_i$. 


An \emph{atomic choice} is a triple $(\ic_i,\theta_j,k)$ where $\ic_i$ is the $i$-th \ac{PIC}, $\theta_j$ is a substitution for the variables in $\ic_i$,  and $k\in \{0,1\}$ indicates whether $\ic_i\theta_j$ is chosen to be included in a \world\ ($k$ = 1) or not ($k$ = 0). 
A \emph{composite choice} $\WorldSet$ is a consistent set of atomic choices, where  \emph{consistent} means that  $\WorldSet$ does not contain two atomic choices  $(\ic_i,\theta_j,k), (\ic_i,\theta_j,m)$ with $k\neq m$, i.e.,  $(\ic_i,\theta_j,k)\in\WorldSet, (\ic_i,\theta_j,m)\in \WorldSet\Rightarrow k=m$ (only one decision for each instance of each \ac{PIC}).

The probability of a  composite choice $\WorldSet$  is 
\begin{equation*}
P(\WorldSet)=\prod_{(\ic_i,\theta_j,1)\in \WorldSet}p_i\prod_{(\ic_i,\theta_j, 0)\in \WorldSet} (1-p_i)\label{eq:pe}
\end{equation*}
where $p_i$ is the probability associated with \ac{PIC} $\ic_i$.

A \emph{selection} $\sigma$ is a \emph{total} composite choice, i.e., it contains an atomic choice for every instance up to renaming of every \ac{PIC}
of the \sciffprob\ program.  
Given a selection $\sigma$, the \world\ $\World_\sigma$ is defined as
$\World_\sigma=\{\ic_i\theta_j|(\ic_i,\theta_j,1)\in \sigma\}$. 
Let us indicate with 
$\mathcal{S}_T$ the set of all selections and with $\SetWorld$ the set of all \world s.
The probability of a \world\ $\World_\sigma$  is $P(\World_\sigma)=P(\sigma)$. 
$P(\World_\sigma)$ is a probability distribution over \world s, i.e., $\sum_{\World\in \SetWorld}P(\World)=1$.

Given a \world\ \World, the conditional probability  that a goal $G$ is satisfied in the \world\ is  defined as $P(G|\World)=1$ if
$\exists \Delta \langle \SOKB, \World, \AbducibleSet \rangle \models^{\Delta} G$
and 0 otherwise.
Given a goal $G$,  its probability 
$P(G)$ can be defined by marginalizing the joint probability of the goal and the \world s:
\begin{eqnarray}
	P(G)&=&\sum_{\World\in \SetWorld}P(G,\World)=\sum_{\World\in \SetWorld} P(G|\World)P(\World)=\sum_{\World\in \SetWorld: \World \models G}P(\World)\notag
\end{eqnarray}
Therefore, the probability of a goal $G$ can be computed by summing the probability of the \world s where the goal is true.

Given a goal $G$ to solve, an \emph{\worldset} is a composite choice $\WorldSet$ for $G$ such that $G$ is entailed  by every 
\world\ of $\omega_\WorldSet$, where $\omega_\WorldSet=\{\World_\sigma|\sigma \in \mathcal{S}_T, \sigma \supseteq \WorldSet\}$
is the set of \world s compatible with $\WorldSet$.

We also define the set of \world s identified by a \setworldset\  $\SetWorldSet$ as
$\omega_\SetWorldSet=\bigcup_{\WorldSet\in \SetWorldSet}\omega_\WorldSet$.
A \setworldset\  $\SetWorldSet$ is \emph{covering} with respect to $G$ if every \world\ $\World\in\SetWorld$ in which $G$ is entailed is such that $\World\in\omega_\SetWorldSet$.
Two \worldset s $\WorldSet_1$ and $\WorldSet_2$ are \emph{incompatible} if their union is inconsistent, i.e., given a \ac{PIC} \ic\ and a substitution $\theta$, the \worldset s $\WorldSet_1=\{(\ic,\theta,1)\}$ and $\WorldSet_2=\{(\ic,\theta,0)\}$ are incompatible. A set \SetWorldSet\ is pairwise incompatible if for all $\WorldSet_1\in\SetWorldSet$, $\WorldSet_2\in\SetWorldSet$, $\WorldSet_1\neq\WorldSet_2$ implies $\WorldSet_1$ and $\WorldSet_2$ are incompatible.
The probability of a pairwise incompatible \setworldset\ \SetWorldSet\ is defined as $P(\SetWorldSet)=\sum_{\WorldSet\in \SetWorldSet}P(\WorldSet)$. 

\begin{definition}[Probabilistic Abductive Answer]
Given an $ALP_p$,
	$(\Delta,\theta)$ is a {\em probabilistic abductive answer} for a goal $G$ in the explanation $\WorldSet$ if 
	$$
	\begin{array}{rrcl}
&		\SOKB \cup \Delta & \models & G\theta \\
&		\SOKB \cup \Delta & \models & \icsetnonprob\\
		\forall (ic_i,\theta_j,1) \in \WorldSet, & \SOKB \cup \Delta & \models & ic_i\theta_j \\
	\end{array}
	$$
%
	In such a case, we write $ALP_p \models^\Delta_\WorldSet G$.
\end{definition}

\begin{example}[continues=exa:running]
We are now able to refine the previous example by associating a probability to the \ac{IC}
saying that if a person {\tt P} enters a house {\tt H},  (s)he must have the keys: such statement
is uncertain, because it might be the case that an unauthorized person  is able to enter a house also without having the keys, although with a lower probability, say 0.3:
\begin{equation}
\label{eq:exampleMurder_IC_Prob}
\tag{{\ensuremath{ic_1'}}}
0.7:: \cabd{enter}{(P,H)} \rightarrow has\_keys(P,H).
\end{equation}

\end{example}

\section{\sciffprob\ Operational Semantics}
\label{sec:iffprob2}

\paragraph{\textbf{\sciff\ Operational Semantics.}}
Before defining the \sciffprob\ operational semantics, we recap the IFF operational semantics \cite{IFF}.

The following is the subset of \sciff\ transitions that are relevant for this work, in a proof-theory style notation.
In the following list, $\fabd{a}$ is an abducible atom, while $p$ and $q$ represent atoms of either abducible or defined predicates, and
$N$ is a set of arguments.
%
\begin{itemize}
\item propagation
$\qquad\frac{\fabd{a}(N) \qquad \fabd{a}(N'), B \rightarrow H}{N=N',B \rightarrow H}$ 
\item unfolding
{\footnotesize\begin{align*}
& \frac{p(N) \qquad p(N') \leftarrow B}{N=N', B}
&& 
\frac{p(N) \rightarrow H \qquad 
 \begin{array}{l}
	p(N') \leftarrow B' \\
	p(N'') \leftarrow B'' \\
	\dots \\
	p(N^k) \leftarrow B^k \\
 \end{array}
}{\begin{array}{l}
	N=N', B' \rightarrow H \\
	N=N'', B'' \rightarrow H \\
	\dots \\
	N=N^k, B^k \rightarrow H \\
 \end{array}
}
\end{align*}}
\item case analysis
{\footnotesize\begin{align*}
\frac{(N=N', B) \rightarrow H}{(N=N'), \quad (B \rightarrow H) } & \qquad
\frac{(N=N', B) \rightarrow H}{ N \neq N'}
\end{align*}}


\item equality rewriting
{\footnotesize\begin{align*}
&\frac{[\exists E][\forall A] A=E}{\theta=\{A/E\}} && \frac{[\exists E][\forall A] A\neq E}{\false}\\
&\frac{[\exists E_1][\exists E_2] E_1=E_2}{\theta=\{E_1/E_2\}} &
&\frac{X=t \qquad \mbox{$t$ does not contain $X$}}{\theta=\{X/t\}}\\
&\frac{X=t \qquad \mbox{$t$ contains $X$}}{\false} && \frac{X \neq t \qquad \mbox{$t$ contains $X$}}{true}\\
&\frac{p(t_1,\dots,t_n)=p(s_1,\dots,s_n)}{t_1=s1, \dots, t_n=s_n} && \frac{p(t_1,\dots,t_n) \neq p(s_1,\dots,s_n)}{t_1 \neq s1 \lor \dots \lor t_n\neq s_n}
\end{align*}
\begin{align*}
& \frac{p(t_1,\dots,t_n)=q(s_1,\dots,s_m) \mbox{ where } p\neq q \lor n \neq m}{\false} &&
\frac{p(t_1,\dots,t_n) \neq q(s_1,\dots,s_m) \mbox{ where } p\neq q \lor n \neq m}{true}
\end{align*}}
where the  substitution $\theta$ is added to the child node. \label{ans:substitution} In case two variables with different quantifiers are unified, the new variable is existentially quantified.
\item logical simplifications

{\footnotesize\begin{equation}
\frac{true \rightarrow A}{A}
\label{eq:iff_logical_equivalence}
\end{equation}
}

$$\begin{array}{ccccc}
\frac{\false \rightarrow A}{true} &
\frac{true \land A}{A} &
\frac{\false \land A}{\false} &
\frac{true \lor A}{true} &
\frac{\false \lor A}{A}
\end{array}$$

\item factoring
$$\begin{array}{cc}
\frac{\fabd{a}(X) \quad \fabd{a}(Y)}{X=Y} &
\frac{\fabd{a}(X) \quad \fabd{a}(Y)}{X\neq Y} 
\end{array}$$
\end{itemize}

\paragraph{\textbf{\sciffprob operational semantics.}}

We provide an extended version of the \sciff\ operational semantics capable to deal with probabilistic integrity constraints.
%
We decorate each probabilistic integrity constraint $ic$ with its original version and  a substitution, initially empty, $\theta_\emptyset$:
$$ic^{ic\theta_\emptyset}$$

The main transition that needs to be updated is the logical equivalence $\frac{\true \rightarrow A}{A}$ (Eq~\ref{eq:iff_logical_equivalence}).
Such a rule continues to exist for non-probabilistic ICs, while for probabilistic ones it is replaced by the following two rules:




\begin{eqnarray}
\frac{\left(p::true \rightarrow A \right)^{ic_i\theta_j} \qquad  \WorldSet \qquad (ic_i,\theta_j,0) \not\in  \WorldSet}{
	 A , \qquad \WorldSet \cup \{(ic_i,\theta_j,1)\}}
\label{eq:trans:prob_logical_equivalence_add}
\\
\frac{\left(p::true \rightarrow A \right)^{ic_i\theta_j} \qquad  \WorldSet \qquad (ic_i,\theta_j,0) \not\in  \WorldSet}{
	\WorldSet \cup \{(ic_i,\theta_j,0)\})}
\label{eq:trans:prob_logical_equivalence_remove}
\end{eqnarray}

Intuitively, when the body of a \ac{PIC} is proven true, its consequences are not propagated in all possible \worldset s, but, instead,
two alternative \worldset s are generated: one in which the \ac{PIC} is assumed to hold and its consequences are propagated, while in the other
the \ac{PIC} is not assumed to hold and the consequence $A$ is not derived.


The other transitions are unmodified, except for the decoration explained earlier, e.g., 
Probabilistic Propagation becomes
$$\frac{
  a(X) \qquad \left( p :: a(Y),B \rightarrow H \right)^{ic\theta}
}{
  \left( p :: X=Y,B \rightarrow H \right)^{ic\theta}
 }
$$



\begin{definition}[Successful derivation]
Each node of a derivation has the form
$$\langle R, \Delta, PSIC, \WorldSet, \theta \rangle$$
where the goal still to be proven is partitioned into the set $R$ (Resolvent) of non-abducible literals and the set $\Delta$ of abducible ones, $PSIC$ is a set of (probabilistic and non-probabilistic) ICs that must hold, 
$\WorldSet$ is the current \worldset, and
$\theta$ is the current substitution.

A {\em derivation is successful} for a goal $G$ in an $ALP = \triple{\SOKB}{\icset}{\AbducibleSet}$ with \worldset\ 
\WorldSet\
if it starts from a node
%
$$\langle G , \emptyset, IC,  \emptyset, \emptyset \rangle$$
and terminates in a node
%
$$\langle \emptyset, \Delta, PSIC, \WorldSet, \theta \rangle$$
%
where no transition is applicable.
If such a derivation exists, we write
$$ALP \derives_{\WorldSet}^\Delta G\theta$$
\end{definition}

\begin{definition}[Computed Probability of a goal]\label{def_prob_g}
Let $\SetWorldSet = \{\WorldSet | ALP \derives_{\WorldSet}^\Delta G \}$ be the \emph{covering} pairwise incompatible set of \worldset s for the goal $G$, 
the computed probability of $G$ is $P(G)=P(\SetWorldSet)=\sum_{\WorldSet\in \SetWorldSet}P(\WorldSet)$.

\end{definition}

\section{Soundness and Completeness}
\label{sec:sound-compl}

\begin{theorem}[Soundness]
If there exists a successful derivation $ALP \derives^\Delta_\WorldSet G\theta$,
then $ALP \models^\Delta_\WorldSet G\theta$.
\end{theorem}
\begin{proof}
Note that, with respect to the original \sciff\ transitions, the new transitions of \sciffprob\ only change the \worldset s, while
all remaining elements of each node of the proof tree remain as in \sciff.
In particular, all non-probabilistic integrity constraints are handled as in \sciff, that is sound 
\cite{IFF},
so for each successful derivation $ALP \derives^\Delta_\WorldSet G$, $KB \cup \Delta \models \icsetnonprob$ holds.

The only modification is probabilistic logical equivalence (Eq~\ref{eq:trans:prob_logical_equivalence_add} and Eq~\ref{eq:trans:prob_logical_equivalence_remove}), in which a new branch is added: 
while in the original \sciff\ proof-procedure,
when the condition of an \ac{IC} is satisfied, the consequence is always added to the node,
for probabilistic \acp{IC} two mutually-exclusive branches are added to the proof-tree.

In one branch, 
Eq~\ref{eq:trans:prob_logical_equivalence_add} is applied and
adds the atomic choice $(ic,\theta,1)$ to $\WorldSet$. 
Note that no transition removes elements from 
\WorldSet.
In such a branch, the consequent of $ic\theta$ is added to the resolvent and
all the following computation treats $ic\theta$ as 
a non-probabilistic integrity constraint in \sciff; this proves that 
for each $(\ic,\theta,1)$ that is added to $\WorldSet$,
$KB \cup \Delta \models ic\theta$.
%

In the other branch, 
Eq~\ref{eq:trans:prob_logical_equivalence_remove} is applied and 
 the atomic choice $(ic,\theta,0)$ is added to $\WorldSet$.
Notice that no transition can add an atomic choice $(ic,\theta,k)$ to $\WorldSet$ if $(ic,\theta,0)$ was already in $\WorldSet$ 
(the set $\WorldSet$ is always consistent).
This proves that $(ic,\theta,1)$ will not be added in the following part of the derivation to $\WorldSet$.
\end{proof}

%
%
%

\begin{theorem}[Completeness]
If $ALP \models^\Delta_{\WorldSet} G\theta$
then there exists a successful derivation 
$ALP \derives^{\Delta'}_{\WorldSet'} G\theta$,
where $\Delta' \subseteq \Delta$ and 
$\forall(ic_i,\theta_j,1)\in \WorldSet'$, $(ic_i,\theta_j,1)\in \WorldSet$.
%
\end{theorem}
\begin{proof}
Suppose that $ALP \models^\Delta_{\WorldSet_*} G$, i.e.,  there is a probabilistic abductive answer for a goal $G$, in \aworldset\ $\WorldSet_*$.
We leverage on the completeness theorem of the IFF proof-procedure \cite{IFF},
crafting an IFF program in such a way that a \sciffprob\ derivation can be built with simple replacements of transitions from the IFF derivation.

Let $IC(\WorldSet^+_*) = \{ic_i\theta_j|(ic_i,\theta_j,1)\in E_*\}$;
consider the (non-probabilistic) 
$$ALP_{\not{p}} = \triple{\SOKB}{\stripProb{(IC(\WorldSet^+_*) \cup \taut{\icsetprob \setminus IC(\WorldSet^+_*)})} \cup \icsetnonprob }{\AbducibleSet}$$
where $\stripProb{X} = \{ ic | p::ic \in X\}$  strips the probability annotation from a set of probabilistic \acp{IC} $X$
and $\taut{X} = \{ Body \rightarrow true | Body \rightarrow Head \in X\}$ replaces the $Head$ of a set of \acp{IC} with $true$; 
clearly the last operation
produces implications that are tautologies, and are useless from a logical viewpoint, but that help us
build  the \sciffprob\ derivation.

From the completeness theorem of the IFF \cite{IFF}, we have that there exists an IFF derivation
$ALP_{\not{p}} \derives^{\Delta'} G$ where $\Delta' \subseteq \Delta$.
We build a successful \sciffprob\ derivation
that mimics the IFF derivation,  where Eq~\ref{eq:trans:prob_logical_equivalence_add} is applied to \acp{IC} in $IC(\WorldSet^+_*)$,
and in which Eq~\ref{eq:trans:prob_logical_equivalence_remove} is applied to  \acp{IC} in its complement, $\icsetprob \setminus IC(\WorldSet^+_*)$.

In the IFF derivation, the logical equivalence Eq~\ref{eq:iff_logical_equivalence} is possibly applied to $\stripProb{\taut{\icsetprob \setminus IC(\WorldSet^+_*)}}$; we substitute each application with the application of Eq~\ref{eq:trans:prob_logical_equivalence_remove}.
Such applications have no other consequences beside adding elements 
$(ic_i,\theta_j,0)$ to the \WorldSet\ set, where $\ic_i\theta_j \in \icsetprob \setminus IC(\WorldSet^+_*)$.

Since \sciffprob\ may also have to apply propagation,
unfolding and case analysis to the \acp{IC} in $\icsetprob \setminus IC(\WorldSet^+_*)$, adding the set
$\stripProb{\taut{\icsetprob \setminus IC(\WorldSet^+_*)}}$ to the integrity constraints of the IFF program ensures that also the IFF derivation applies such transitions when necessary.

In the IFF derivation, the logical equivalence $\frac{true \rightarrow A}{A}$ is (possibly) applied to \acp{IC} in $\stripProb{IC(\WorldSet^+_*)}$.
We substitute each application of Eq~\ref{eq:iff_logical_equivalence} with the application of Eq~\ref{eq:trans:prob_logical_equivalence_add}.
Note that Eq~\ref{eq:trans:prob_logical_equivalence_add} is applicable, because the set 
$\WorldSet$ 
is empty in the initial node,
and the only transition that adds
atomic choices of the type $(ic_i,\theta_j,0)$
to \WorldSet\ is Eq~\ref{eq:trans:prob_logical_equivalence_remove}, that adds only elements of  $\icsetprob \setminus IC(\WorldSet^+_*)$.
Clearly the result in the resolvent is the same as for the original IFF derivation, since both Eq~\ref{eq:iff_logical_equivalence} and 
Eq~\ref{eq:trans:prob_logical_equivalence_add} add the consequence $A$ of the implication to the resolvent.
By construction, the set of atomic choices $(ic_i,\theta_j,1)$ added in this way to \WorldSet\ is always a subset of $IC(\WorldSet^+_*)$.
\end{proof}

\section{Implementation}
\label{sec:impl}

\newcommand{\origsciff}{{\ensuremath{\mathcal{S}\mbox{CIFF}}}}

\subsection{Recap: CHR implementation of the \origsciff\ proof-procedure}
We implemented \sciffprob\ leveraging on the implementation of the
\origsciff\ proof-procedure. \origsciff\ \cite{sciff-tocl} is an extension of the \sciff\ proof procedure \cite{IFF}
that also features  constraints (\textit{\`a la} \ac{CLP}) and universally quantified abducibles.
In this work, we extend to the probabilistic case only the IFF sub-language and leave for future work the probabilistic extension of the other parts of \origsciff.

The \origsciff\ proof-procedure was implemented \cite{sciff_implementation} in \ac{CHR} \cite{DBLP:journals/fuin/Fruhwirth20}. \ac{CHR} is a rewriting system originally
developed to implement new \ac{CLP} constraint solvers, and then employed also as a language for a wide variety of applications.
A set of atoms are declared as CHR constraints; according to the CLP operational semantics, when a CHR constraint is selected (e.g., during SLD-resolution), it is moved to a constraint store. CHR rules transform the constraint store, hopefully simplifying it. There exist two main types of CHR rules: {\em propagation} and {\em simplification} rules.

A {\em Propagation } rule has the form
\vspace{-0.6em}
$$c_1, c_2, \dots, c_n \Rightarrow Guard | Body$$
where $c_1, \dots, c_n$ are CHR constraints, $Guard$ and $Body$ are  Prolog goals.
The meaning is that $Body$ is a logical consequence of the conjunction $c_1 \land \dots \land c_n$ provided that $Guard$ is true.
Operationally, when the set of constraints $c_1, \dots, c_n$ are in the constraint store, the guard is evaluated; in case it is true, the $Body$ is executed.

A {\em Simplification} rule has the form
\vspace{-0.6em}
$$c_1, c_2, \dots, c_n \Leftrightarrow Guard | Body$$
Its declarative reading is that $Body$ is equivalent to the conjunction $c_1 \land \dots \land c_n$ provided that  $Guard$ is true.
Operationally, when the set of constraints $c_1, \dots, c_n$ is in the store, and the guard is true, the constraints  $c_1, \dots, c_n$ are removed from the store and  the $Body$ is executed.

In the \origsciff\ implementation, 
each abducible atom $\cabd{a}{(X)}$ is mapped to a CHR constraint \texttt{abd(a(X))}, and each \ac{IC} $Body \rightarrow Head$ is mapped to a CHR constraint \texttt{ic(Body,Head)}.
Each transition in the operational semantics is mapped to a \ac{CHR} rule.
For example, transition {\em propagation} is mapped to the CHR simplification rule (CHR expert readers will actually recognize it as a simpagation rule)
$$abd(A),\ ic([abd(B)|T],H)\ \Leftrightarrow\ abd(A),\ ic([A=B|T],H),$$
and transition {\em case analysis} is mapped to
\vspace{-0.6em}
$$ic([A=B|T],H)\ \Leftrightarrow\ A=B,\ ic(T,H) \ ;\  A \neq B$$
where the semicolon is Prolog's OR and $\neq$ is a dis-unification constraint.
Logical equivalence $\frac{true \rightarrow G}{G}$ is mapped to the simplification rule
\vspace{-0.6em}
\begin{equation}
ic([],Head) \Leftrightarrow Head.
\label{eq:logical_equivalence_CHR_sciff}
\end{equation}

\subsection{Implementation of \sciffprob}

In order to implement  probabilistic reasoning, 
we add a new CHR constraint that represents the current \worldset:
\vspace{-1em}
$$\worldsetCHR(\WorldSet,P)$$
means that, in the current derivation branch, the \worldset\ is $\WorldSet$, and has probability $P$.
The $\WorldSet$ parameter is  a collection (e.g., list) of triples $(ic_i,\theta_j,k)$, holding the integrity constraint $ic_i$, the substitution that made its body true, $\theta_j$, and the integer $k$ representing a Boolean value explaining whether $ic_i\theta_j$ belongs to  \WorldSet\ or not.

The CHR constraint representing \acp{IC} now requires additional information for probabilistic \acp{IC}: 
unsurprisingly the probability $P$
needs to be added to the parameters, and also the original version $ic_i$ of the integrity constraint together with the $\theta_j$ substitution that binds the variables in the body is stored. The new CHR constraint
representing a probabilistic implication
is
\vspace{-0.6em}
$$ic(Body,Head,ic_i\theta_j,P).$$

\begin{figure}
\begin{algorithmic}[1]
\State $ic([],Head,ic_i\theta_j,P_{ic}),$\ $\worldsetCHR(\WorldSet,P_\WorldSet)$
 $\Leftrightarrow$
\State\label{line:ic_already_in_world} $( \quad (ic_i,\theta_j,\_) \in \WorldSet,$
\State $\phantom{(}\quad \worldsetCHR(\WorldSet,P_\WorldSet)$
\State\label{line:ic_not_in_world} $; \quad (ic_i,\theta_j,\_) \not\in \WorldSet,$
\State\label{line:add_ic} $\phantom{(}\quad	 ( \quad \worldsetCHR(\WorldSet\cup \{(ic_i,\theta_j,1)\},P_\WorldSet P_{ic}),$
\State $\phantom{(}\quad	 \phantom{(}\quad  Head$
\State\label{line:remove_ic} $\phantom{(}\quad	 ; \quad \worldsetCHR(\WorldSet\cup \{(ic_i,\theta_j,0)\},P_\WorldSet (1-P_{ic}))$
\State $\phantom{(}\quad )$
\State $).$
\end{algorithmic}
\caption{\label{fig:implementation_logical_equivalence} Implementation of the probabilistic logical equivalence rule in CHR.}
\end{figure}

We extend the logical equivalence rule $\frac{true \rightarrow A}{A}$ according to the new operational semantics (Figure~\ref{fig:implementation_logical_equivalence}).
Such logical equivalence is applied when (due to successive applications of other transitions) the body of an integrity constraint $ic_i$ is proven true, for a given substitution $\theta_j$ of the variables in the body.
Simplification rule \ref{eq:logical_equivalence_CHR_sciff}
is extended to consider three cases in disjunction.
In the first (line~\ref{line:ic_already_in_world}), it is imposed that $ic_i\theta_j$ belongs to the current \worldset\ (i.e., a unification is imposed such that $\ic_i\theta_j$
unifies with at least one member of the \worldset, similarly to the {\tt member} predicate in standard Prolog);
 in such a case, the $ic$ has already been propagated with exactly the same substitution, so there is no need to re-propagate its consequences, nor to change the current \worldset\ and its probability.
Otherwise (line~\ref{line:ic_not_in_world}), it is imposed that $ic_i\theta_j$ does not belong to the current \worldset\ (i.e., a dis-unification constraint is imposed between $ic_i\theta_j$ and all members of the \worldset).
In such a case, two alternative branches are opened, i.e., 
one in which we consider \aworldset\ that includes the integrity constraint $ic_i\theta_j$ (line~\ref{line:add_ic}) and one in which $ic_i\theta_j$ is considered removed (line~\ref{line:remove_ic}).


\subsection{Computation of the Probability of Goals}

The \sciffprob proof-procedure returns the covering set \SetWorldSet\ of all the \worldset s $\worldsetCHR/2$ for a goal $G$. As seen in Definition~\ref{def_prob_g}, if \SetWorldSet\ is also pairwise incompatible, then $P(G)=P(\SetWorldSet)$.
However,  a covering \setworldset\ $\SetWorldSet$ for a goal $G$ is not guaranteed to be pairwise incompatible. 


Thus, we  associate a Boolean random variable $X_{ij}$ to each instance of \ac{PIC} $\ic_i\theta_j$. In this way, an atomic choice $(\ic_i,\theta_j,1)$  corresponds to $X_{ij}$ assuming value $true$.
The variables  $\mathbf{X}=\{X_{ij}|(\ic_i,\theta_j,k) \in \WorldSet, \WorldSet \in \SetWorldSet\}$ are pairwise independent and the probability that $X_{ij}$ takes value 1 is $p_i$, the probability associated with the $i$-th \ac{PIC}.

Given a covering \setworldset\ $\SetWorldSet$ for a query $G$, each \world\ where the query is true corresponds to an assignment of $\mathbf{X}$ for which the following Boolean function takes value 1:
\begin{equation}
	f_\SetWorldSet(\mathbf{X})=\bigvee_{\WorldSet\in \SetWorldSet}\bigwedge_{(\ic_i,\theta_j,1)\in \WorldSet}X_{ij}\bigwedge_{(\ic_i,\theta_j,0)\in \WorldSet}\overline{X_{ij}}
\end{equation}
Thus, we can compute the probability of $G$ by computing the probability that $f_\SetWorldSet(\mathbf{X})$ takes value 1.
This formula is in Disjunctive Normal Form (DNF) but we cannot compute $P(f_\SetWorldSet(\mathbf{X}))$ by summing the probability of each individual explanation because the different explanations may not be mutually disjoint.
To solve the problem, we can apply knowledge compilation to the propositional formula $f_\SetWorldSet(\mathbf{X})$ \cite{DBLP:journals/jair/DarwicheM02} in order to translate it into a target language that
allows the computation of the probability in polynomial time. A target language that was found to give good performances is the one of \acp{BDD}. 

A \ac{BDD} for a  function of Boolean variables is   
a rooted graph that has one level for each Boolean variable. 
A node $n$ in a \ac{BDD} has two children: one corresponding to the 1 value of the variable associated with $n$, indicated with $child_1(n)$, and one corresponding to the 0 value of the variable, indicated with $child_0(n)$. 
The leaves store either 0 or 1.
Given values for all the variables, a \ac{BDD} can be used to compute the value of the formula by traversing the graph starting from the root, following the edges corresponding to the variables values and returning the value associated to the leaf that is reached. 
For instance, Figure \ref{dd} shows a \ac{BDD} for the function $f(\mathbf{X})=(X_{11}\wedge X_{21})\vee (X_{12}\wedge X_{21}).$
%
\begin{figure*}
	\begin{minipage}[b]{.35\textwidth}
			$$
			\xymatrix@R=5mm@C=4mm
			{ X_{11} & &*=<25pt,15pt>[F-:<3pt>]{n_1}
				\ar@/_/@{-}[ldd] \ar@/^/@{--}[dr]\\ 
				X_{12}  & & &*=<25pt,15pt>[F-:<3pt>]{n_2} 
				\ar@/_/@{-}[dll]\ar@{--}[dd] 
				\\
				X_{21}& *=<25pt,15pt>[F-:<3pt>]{n_3}
				\ar@{-}[d] \ar@/^/@{--}[drr]  \\
				&*=<25pt,15pt>[F]{1} &&*=<25pt,15pt>[F]{0}}
			$$
			
			\captionof{figure}{\ac{BDD} for function $f(\mathbf{X})$.}
			\label{dd}
	
	\end{minipage}
	\hfill
	\begin{minipage}[b]{.62\textwidth}
	\footnotesize
	\begin{algorithmic}[1]
				\Function{Prob}{$node$, $pMap$}
				\State Input: a BDD node $node$ and a (node,prob) map $pMap$
				\State Output: the probability of the formula associated to $node$
				\If{$node$ is a terminal}
				\State return $value(node)$\Comment{$value(node)$ is 0 or 1}
				\ElsIf{$node$ is in $pMap$}
				\State return $pMap[node]$ \Comment{The probability of $node$}
				\Else
				\State let $X$ be $v(node)$ \Comment{$v(node)$ is the variable associated to $node$}
				\State $P_1\gets$\Call{Prob}{$child_1(node)$}
				\State $P_0\gets$\Call{Prob}{$child_0(node)$}
				\State add $(node,P(X)\cdot P_1+(1-P(X))\cdot P_0)$ to $pMap$
				\State return $P(X)\cdot P_1+(1-P(X))\cdot P_0 $
				\EndIf
				\EndFunction
			\end{algorithmic}
			\captionof{figure}{Function that computes the probability of a 
				BDD.\label{prob}}
	\end{minipage}
\end{figure*}
A \ac{BDD} performs a Shannon expansion of the Boolean formula $f_\SetWorldSet(\mathbf{X})$, so that if $X$
is the variable associated to the root level of a \ac{BDD}, the formula $f_\SetWorldSet(\mathbf{X})$ can be represented as
$f_\SetWorldSet(\mathbf{X})=X\wedge f_\SetWorldSet^X(\mathbf{X})\vee \overline{X}\wedge f_\SetWorldSet^{\overline{X}}(\mathbf{X})$
where $f_\SetWorldSet^X(\mathbf{X})$ (respectively, $f_\SetWorldSet^{\overline{X}}(\mathbf{X})$) is the formula obtained by $f_\SetWorldSet(\mathbf{X})$ by setting $X$ to 1
(resp., 0). Now the two disjuncts are mutually exclusive and the probability of $f_\SetWorldSet(\mathbf{X})$ can be computed as
$P(f_\SetWorldSet(\mathbf{X}))=P(X)P( f_\SetWorldSet^X(\mathbf{X}))+(1-P(X))P(f_\SetWorldSet^{\overline{X}}(\mathbf{X}))$. 
Figure \ref{prob} shows the function \textsc{Prob} that implements the dynamic programming algorithm of \cite{DBLP:conf/ijcai/RaedtKT07} for computing the probability of a formula encoded as a \ac{BDD}.

\section{Application examples}
\label{sec:app_ex}
We can now show how probabilities are computed in the running example; we then show
a second example taken from the literature, to show the versatility of our approach.

\begin{example}[continues=exa:running]

Here $\icsetprob=\{\ref{eq:exampleMurder_IC_Prob}\}$,  $\icsetnonprob=\emptyset$ and the set of abducible predicates is $\AbducibleSet =\{\cabd{enter}{/2}, \cabd{killed}{/2}\}$, since they are not known.

Two atoms may make  the body of \ref{eq:exampleMurder_IC_Prob} true, one stating that $\texttt{husband}$ has the keys of ${\texttt{house1}}$ and the other stating that he holds the keys of ${\texttt{house2}}$; these will correspond to the substitutions 
$\theta_1=\{\texttt{P/husband},\texttt{H/house1}\}$ 
and 
$\theta_2= \{\texttt{P/husband},\texttt{H/house2}\}$.
For each instantiation that makes the body true, exactly one of the two transitions in Eq. \ref{eq:trans:prob_logical_equivalence_add} and \ref{eq:trans:prob_logical_equivalence_remove} is applicable.
In this way, four explanations 
are generated in alternative branches:
$E_1 = \{(\ref{eq:exampleMurder_IC_Prob}, \theta_1,1),(\ref{eq:exampleMurder_IC_Prob},\theta_2,1)\}$ with $P(E_1)=0.7^2=0.49$,
$E_2 = \{(\ref{eq:exampleMurder_IC_Prob},\theta_1,1),(\ref{eq:exampleMurder_IC_Prob},\theta_2,0)\}$ with $P(E_2)=0.7\cdot(1-0.7) =0.21$,
$E_3 = \{(\ref{eq:exampleMurder_IC_Prob},\theta_1,0),$ $(\ref{eq:exampleMurder_IC_Prob},\theta_2,1)\}$ with  $P(E_3)=(1-0.7)\cdot 0.7 =0.21$ 
and
$E_4 = \{(\ref{eq:exampleMurder_IC_Prob}, \texttt{\{P/X,H/house1\}},0)$, $(\ref{eq:exampleMurder_IC_Prob},\texttt{\{P/X,H/house2\}},0)\}$ with $P(E_4)=(1-0.7)^2=0.09$.

The sets of explanations where the goal is true are:
\begin{description}
\item [$\SetWorldSet_1$]$=\{E_1,E_2,E_3\}$ with probability 
 $P(\SetWorldSet_1)= \sum_{i=1}^3P(E_i) = 0.91$, probabilistic abductive answer  
{\small$\Delta = \{\cabd{enter}{(husband,house1)},\cabd{enter}{(husband,house2)},$ $\cabd{killed}{(husband,woman)}\}$ and $\theta= \{\texttt{M/husband}\}$}.
 This solution (the most likely) states that the husband was the killer with a chance of 91\%. Explanation $E_1$ indicates that the goal is true (the husband was the killer) if  the husband had entered both houses since he had the keys for both.  The second and third explanations represent the fact that one instantiation of the PIC is not considered, so with probability 0.3 one can
enter a house even if (s)he does not have the keys.

\item [$\SetWorldSet_2$]$=\{E_4\}$
with  $P(\SetWorldSet_2)=P(E_4)=(1-0.7)^2=0.09$; note that both instances of $\ref{eq:exampleMurder_IC_Prob}$ are relaxed $(k=0)$ so the complement of 0.7 must be used. 
In this case the probabilistic abductive answer is
{\small$(\Delta,\theta) = (\{\cabd{enter}{(M,house1)},$ $\cabd{enter}{(M,house2)},$ $\cabd{killed}{(M,woman)}\},$ $
\emptyset)$}. 
This solution (much less probable) states that some unknown person entered the two houses and committed the murder with a chance of 9\%.
\end{description}

\sciffprob\ provided two reasonable explanations.
In particular,  the second abductive explanation, obtained through non-ground abductive reasoning, is that some
person, unknown to the knowledge base, could have perpetrated the murder.
It is interesting to know that this is exactly the explanation suggested by the husband during the trial:
he pleaded not guilty and suggested that some other person could have entered his house, killed the woman, and carried the pillow to the second house.
Although logically possible, such explanation was considered unlikely by the judges, that sentenced the husband guilty. 
Nevertheless, as a human brain was able to hypothesize the existence of an external person entering both houses and killing the victim,
also the \sciffprob\ proof-procedure was able to produce such hypothesis. And as a human brain judged such hypothesis unrealistic,
also the proof-procedure assigned it a low probability.
Of course, this is a simplification of all the evidence collected during the trial, that we cannot report here due to lack of space.

It is also interesting to vary the probability $q$ associated with \ref{eq:exampleMurder_IC_Prob} to see when  $\SetWorldSet_2$ becomes the most probable explanation. 
Now the probabilities are:
$P(E_1)=q^2$,
$P(E_2)=q\cdot(1-q) $,
$P(E_3)=q\cdot(1-q)$ 
and
$P(E_4)=(1-q)^2$.
In order to have $P(E_1)+P(E_2)+P(E_3) < P(E_4)$,
it must be that $q< 1-\frac{\sqrt{2}}{2}\simeq 0.293$.
\end{example}

\begin{example}

\citeN{DBLP:series/lncs/Christiansen08} proposes an example of power supply network diagnosis (here adapted to \sciffprob\ syntax);
a power plant \verb|pp| provides electricity to a few villages \verb|v|$_i$ through directed wires \verb|w|$_i$.
The network structure is described by a set of \verb|edge/3| facts (please refer to the figure of the network in Section 7.1 in \cite{DBLP:series/lncs/Christiansen08}):
\begin{lstlisting}
edge(w1, pp, n1).    edge(w4, n3, v3).    edge(w7, n3, v2).
edge(w2, n1, n2).    edge(w5, n1, n4).    edge(w8, n4, v4).
edge(w3, n2, n3).    edge(w6, n2, v1).    edge(w9, n4, v5).
\end{lstlisting}


The fact that a given point (the power plant, a node or a village) in the network has no electricity is described in \cite{DBLP:series/lncs/Christiansen08} by means of the \verb|hasnopower/1| predicate. 
\begin{lstlisting}
hasnopower(pp) $\leftarrow$ $\cabd{down}{(pp)}$. 
hasnopower(N2) $\leftarrow$ edge(W,_,N2), $\cabd{down}{(W)}$.
hasnopower(N2) $\leftarrow$ edge(_,N1,N2), hasnopower(N1).
\end{lstlisting}
while the opposite situation is described by:
\begin{lstlisting}
haspower(pp) $\leftarrow$ $\cabd{up}{(pp)}.$
haspower(N2) $\leftarrow$ edge(W,N1,N2), $\cabd{up}{(W)}$, haspower(N1). 
\end{lstlisting}

The clauses and facts  mentioned above represent the \SOKB\ of the \sciffprob\ program.
In the proof-procedure  by \citeN{DBLP:series/lncs/Christiansen08}, probabilities are associated to abducible atoms, and he defines
the abducibles \cabd{up}{/1} with probabilities 0.9 (every instance of $\cabd{up}(X)$ has the same probability) and \cabd{down}{/1}  with 0.1.
As noted earlier, in \sciffprob\ we can use \acp{PIC} to associate probabilities to single abducibles, using the complementary probability:
\begin{lstlisting}
$ic_1$ = 0.1::$\cabd{up}{(X)}$ $\rightarrow$ false.
$ic_2$ = 0.9::$\cabd{down}{(X)}$ $\rightarrow$ false.
\end{lstlisting}
Finally, the program in \cite{DBLP:series/lncs/Christiansen08} also includes one \ac{IC} to state that no node may be up and down at the same time:
\begin{lstlisting}
$ic_3$ = $\cabd{up}{(X)}$ $\wedge$ $\cabd{down}{(X)}$ $\rightarrow$ false.
\end{lstlisting}
Here $\icsetprob=\{ic_1,ic_2\}$,  $\icsetnonprob=\{ic_3\}$ and the set of abducible predicates is $\AbducibleSet =\{\cabd{up}{/1}, \cabd{down}{/1}\}$.

Given the goal:
\begin{lstlisting}
    $G$ = hasnopower(v1),hasnopower(v2),hasnopower(v3),
        hasnopower(v4),hasnopower(v5)
\end{lstlisting}
$P(G) = 0.199695$ is returned as the probability that no village has electricity, as 
the sum of the probability of the worlds where the goal is true, throughout the application of function \textsc{Prob} of Fig.~\ref{prob}. In total, 1600 worlds are found, corresponding to the different combinations of failures of wires/power plant.
As also computed by \citeN{DBLP:series/lncs/Christiansen08}, the two most probable worlds have both probability $0.1$ and are respectively identified by the explanations:
\begin{description}
\item[$E_1 
= \{(ic_2, \texttt{X/pp},0)\}$] having probability $P(E_1)=(1-0.9)=0.1$; the complement of $0.9$ is taken as $ic_2$ with substitution $\theta_1=\{\texttt{X/pp}\}$ (i.e., \texttt{$\cabd{down}{(pp)}$ $\rightarrow$ false.}) is not included ($k=0$). The probabilistic abductive answer is  $(\Delta,\theta) = (\{\cabd{down}{(pp)}\}, \theta_1)$;
\item[$E_2 
= \{(ic_2, \texttt{X/w1},0)\}$] having probability $P(E_2)=(1-0.9)=0.1$; here $ic_2$ with substitution $\theta_2=\{\texttt{X/w1}\}$  is not included. The probabilistic abductive answer is  $(\Delta,\theta) = (\{\cabd{down}{(w1)}\}, \theta_2)$.
\end{description}
$E_1$ indicates that  the goal is true (no village receives power)  if the power plant is down with probability 0.1, while $E_2$ 
if the main wire (\texttt{w1}) is down with probability 0.1.
\end{example}

\section{Experiments}
\label{sec:exp}
To test how our approach reacts with an increasing number of worlds, given an integer $n \geq 1$, we considered the \ac{ALP} containing the following ICs for $1\leq i \leq n$:
\begin{lstlisting}
	0.6::$\texttt{\textit{b}}_{i-1}$(X)  $\rightarrow$ $\texttt{\textit{p}}_i$(X) $\wedge$ $\texttt{\textit{q}}_i$(X).
	0.6::$\texttt{\textit{p}}_i$(X) $\rightarrow$ $\texttt{\textit{b}}_{i}$(X).
	0.6::$\texttt{\textit{q}}_i$(X) $\rightarrow$ $\texttt{\textit{b}}_{i}$(X).
\end{lstlisting}
where $\texttt{\textit{b}}_{i}\texttt{(X)}$, $\texttt{\textit{p}}_{i}\texttt{(X)}$, and $\texttt{\textit{q}}_{i}\texttt{(X)}$ are abducibles for all $i$.
 For the test purpose, we built 12 \acp{ALP} of increasing size containing the above ICs for $i$ varying from 1 to $n$, with $n\in\{1,2,3,\dots,11,12\}$.
 
These \acp{ALP} present a number of worlds that grows exponentially with $n$. This is due to the fact that $\texttt{\textit{p}}_{i+1}\texttt{(X)}$ and $\texttt{\textit{q}}_{i+1}\texttt{(X)}$ are needed to abduce $\texttt{\textit{b}}_i\texttt{(X)}$. $\texttt{\textit{p}}_{i+1}\texttt{(X)}$ and $\texttt{\textit{q}}_{i+1}\texttt{(X)}$ can either both be true, or one true while the other must be abduced, or both must be abduced. In turn, to
abduce $\texttt{\textit{p}}_{i+1}\texttt{(X)}$ or $\texttt{\textit{q}}_{i+1}\texttt{(X)}$, one must use $\texttt{\textit{b}}_{i+1}\texttt{(X)}$, which may be already abduced or may possibly have to be abduced. These possible cases must be considered for all values assumed by $i$, creating an exponential number of possible ways to prove $\texttt{\textit{b}}_i\texttt{(X)}$.

For the test, we  compute the probability of the goal $G=\texttt{\textit{b}}_0\texttt{(X)}$. To compute the running time, we ran the \sciffprob\ proof procedure 5 times w.r.t. each \ac{ALP} built and we computed the average running time with its standard deviation. 
All the tests have been performed on a Linux machine  equipped with IBM\textsuperscript{\textcopyright} POWER9\textsuperscript{TM} AC922 at 2.6(3.1)GHz, with 256 GB of RAM. We allowed a maximum of 100 GB for the Prolog stack, necessary to hold the worlds and the information about the choice points for the \acp{ALP} with $i \geq 11$.

\begin{table}
	\centering
	\caption{Average running time (in seconds) on 5 executions of the \sciffprob\ proof-procedure and the $\pm$ the standard deviation. ALP size indicates the value $n$ takes in each ALP, i.e., it defines the range $1\leq i \leq n$. OM means Out of Memory.\label{tab:test-res}}
	\footnotesize
	\begin{tabular}{c|c|c}
		ALP  size ($n$) & N. Worlds & Time (s) $\pm$ Std. dev. \\ 
		         1           &     5     &   $0.00133\pm0.00003$    \\
		         2           &    17     &   $0.00387\pm0.00001$    \\
		         3           &    53     &   $0.01228\pm0.00003$    \\
		         4           &    161    &   $0.04042\pm0.00013$    \\
		         5           &    485    &   $0.13582\pm0.00059$    \\
		         6           &   1 457   &   $0.45903\pm0.45903$    \\
		         7           &   4 373   &   $1.91763\pm0.01396$    \\
		         8           &  13 121   &   $6.32130\pm0.02730$    \\
		         9           &  39 365   &   $21.51733\pm0.06072$   \\
		         10          &  118 097  &   $73.78656\pm0.29424$   \\
		         11          &  354 293  &  $333.90058\pm0.78041$   \\
		         12          & 1 062 881 &            OM            \\ 
	\end{tabular} 
\end{table}

Table~\ref{tab:test-res} shows, for each \ac{ALP}, the number of found worlds and the average running time in seconds to compute the probability of the goal $\pm$ the standard deviation. The value OM represents the fact that the proof-procedure went out of the available stack due to the high number of worlds and choice points to maintain in memory.
As one can see, our approach is able to manage very high number of worlds within 1 minute.

\section{Related work}
\label{sec:rw}

Some works explicitly addressed probabilistic abductive reasoning: \citeN{TurliucEtAl2013}  rank explanations in terms of probabilities and investigate the role of integrity constraints. They define a probability distribution over the truth values of each
(ground) abducible, while we set probabilities on integrity constraints.
SOLAR~\cite{DBLP:conf/ijcai/InoueSIKN09} is a system for abductive inference that applies an Expectation Maximization (EM) algorithm  for evaluating  hypotheses obtained from the process of hypothesis generation. 
After generating all minimal explanations, the EM algorithm, working on \acp{BDD}, is used to assign probabilities to atoms in explanations.
Finally, SOLAR computes the probability of each hypothesis to find the most probable one.
EM is also used by~\citeN{conf/ijcai/Raghavan11}, who considers Bayesian Logic Programs (BLPs), and by \citeN{kate2009rj}, who consider Markov Logic Networks (MLNs).  EM is exploited to learn the parameters associated with the model. 

Differently, \citeN{DBLP:journals/ai/Poole93} considers Bayesian networks but focuses on the definition of the language instead of the combination of abductive proof-procedures and statistical learning. Moreover, it imposes assumptions on the type of constraints in order to simplify the procedure. 

\citeN{arvanitis2006abduction} consider Stochastic Logic Programs (SLPs), where abductive reasoning is done by reversing  deduction, i.e., reversing the flow of the proof-procedure. However, this may return wrong conclusions without imposing \textit{ad hoc} constraints in the program.

\citeN{DBLP:conf/cilc/RotellaF13} define new types of probabilistic constraints to guide the search of explanations that are consistent with the constraints, by giving priority to explanations having higher probability to be true.
However, all these approaches do not allow non-ground abduction.
We refer to~\cite{PILP} for a description of BLPs, MLNs, and SLPs.

Abduction is also used in machine learning. For example in \cite{Muggleton2015}, it is used to perform predicate invention and recursive generalisations with respect to a meta-interpreter. In this case, however, values for predicate variables  rather than values for first order variables are abduced.
More recently, \cite{Muggleton2021} proposed to use abduction to infer constraints for learning problems but do not consider existential variables.

\citeN{DBLP:series/lncs/Christiansen08} also implements probabilistic non-ground abduction in \ac{CHR}; the main difference with our work
is that in his work probabilities are associated with abducibles, while in our work they are associated with integrity constraints.
The integrity constraints in his proof-procedure are more limited in syntax, since they can only accommodate abducible predicates and are in the form of denials, while in \sciff\ they can include all types of atoms, and are
in the form of implications, with disjunctions in the head.
Due to the syntactic restrictions on integrity constraints, only a limited form of negation is possible, while in IFF sound negation can be applied to
both abducible and defined atoms.
Differently from our solution, Christiansen adopts a best-first search scheme, in which branches with higher probability are explored before the branches
with lower probability; while the exact probability is only known at the end of the whole search, at each found solution a lower bound is obtained, and it gets more precise as new solutions are found.
On the other hand, a best-first search has higher memory requirements than a depth-first search.

\section{Conclusions}
\label{sec:conc}

We presented a probabilistic abductive logic programming language able to perform abductive reasoning with variables,
and  probabilities attached to constraints.
The need to have probabilistic integrity constraints comes from probabilistic reasoning with many real-life applications, and such integrity constraints may be learned from available data \cite{RigBelZesAlbLam20-ML-IJ}. 
We showed two examples on different domains of abductive reasoning with probabilities in action, showing the usefulness of non-ground abduction and that
our language can also tackle   problems with probabilities attached to abducibles.
Soundness and completeness of the devised proof-procedure have been shown.

Future work concerns considering non-ground probabilities, i.e., variable probabilities attached to integrity constraints and its CHR implementation, as well as probabilistic clauses in the \SOKB, and \ac{CLP} constraints.

\bibliographystyle{acmtrans}

\label{lastpage}
\end{document}